\documentclass[a4paper]{article}
\usepackage{amsthm,amsmath,amssymb}
\usepackage{aliascnt}
\usepackage{framed}
\usepackage{enumerate}
\usepackage{algorithm}
\usepackage{algpseudocode}
\usepackage{pifont}
\usepackage{microtype}
\usepackage{tikz}
\usepackage{authblk}
\usepackage[margin=3.5cm,top=3cm,bottom=3.5cm,footskip=2cm]{geometry}

\newcommand{\nth}{n^{\text{th}}}

\newcommand{\lb}{\left}
\newcommand{\rb}{\right}

\newcommand{\coef}{\bar{c}}
\newcommand{\poly}{\mathrm{poly}}
\newcommand{\twin}{\mbox{\textsf{make twin-free}}}
\newcommand{\solveLin}{\mbox{\textsf{solve linear system}}}
\newcommand{\solveMat}{\mbox{\textsf{solve linear matrix system}}}
\newcommand{\aug}{\mbox{\textsf{augment $M$ with}}}

\newcommand{\findB}{\mbox{\textsf{find basis}}}

\newcommand{\val}{\mbox{\textsc{value}}}
\newcommand{\equ}{\mbox{\textsc{equivalent}}}
\newcommand{\hyp}{\mbox{\textsf{generate hypothesis}}}
\newcommand{\mem}{\mbox{\textsc{membership}}}
\newcommand{\fillMatrix}{\mbox{\textsf{fill matrix}}}
\newcommand{\addBlock}{\mbox{\textsf{add block}}}

\newcommand{\bsb}{\boldsymbol{\mathrm{b}}}

\newcommand{\Aut}{\mathrm{Aut}}

\newcommand{\Hom}{\text{\textsf{hom}}}

\newcommand{\cB}{\mathcal{B}}
\newcommand{\cG}{\mathcal{G}}

\newcommand{\cQ}{\mathcal{Q}}
\newcommand{\cR}{\mathcal{R}}

\newcommand{\bR}{\mathbb{R}}

\newcommand{\bZ}{\mathbb{Z}}

\theoremstyle{plain}
\newtheorem{theorem}{Theorem}
\newtheorem{lemma}[theorem]{Lemma}
\newtheorem{prop}[theorem]{Proposition}
\newtheorem{claim}[theorem]{Claim}
\newtheorem{remark}[theorem]{Remark}
\newtheorem{rem}[theorem]{Remark}
\newtheorem{cor}[theorem]{Corollary}

\begin{document}

\title{On the exact learnability of graph parameters: \\ The case of partition functions
\footnote{This is the complete version of the MFCS 2016 paper.}}

\author{\normalsize Nadia Labai
\thanks{Supported by the National Research Network RiSE (S114), and the LogiCS doctoral program (W1255) funded by the Austrian Science Fund (FWF).}}
\affil{Department of Informatics, Vienna University of Technology, Vienna, Austria\\
  \texttt{labai@forsyte.at}}
	
\author{\normalsize Johann A. Makowsky
\thanks{Partially supported by a grant of Technion Research Authority. 
This work was done [in part] while the author was visiting the Simons
Institute for the Theory of Computing.}}
\affil{Department of Computer Science, Technion - Israel Institute of Technology, Haifa, Israel\\
  \texttt{janos@cs.technion.ac.il}}
	
\date{}

\maketitle

\begin{abstract}
We study the exact learnability of real valued graph parameters $f$ which are known to be 
representable as partition functions which count the number of weighted homomorphisms into a graph $H$ with vertex weights $\alpha$ and edge weights $\beta$.
M. Freedman, L. Lov\'asz and A. Schrijver have given a characterization of these graph parameters in terms of the $k$-connection matrices $C(f,k)$ of $f$. Our model of learnability is based on D. Angluin's model of exact learning using membership and equivalence queries. Given such a graph parameter $f$, the learner can ask for the values of $f$ for graphs of their choice, and they can formulate hypotheses in terms of the connection matrices $C(f,k)$ of $f$.
The teacher can accept the hypothesis as correct, or provide a counterexample consisting of a graph.
Our main result shows that in this scenario, a very large class of partition functions,
the rigid partition functions, can be learned in time polynomial in the size of $H$ and the size of the largest counterexample in the Blum-Shub-Smale model of computation over the reals with unit cost.
\end{abstract}

\newif\ifappendix
\appendixtrue

\section{Introduction}
\newcommand{\bN}{\mathbb{N}}
\newcommand{\bQ}{\mathbb{Q}}

A graph {\em parameter $f: \mathcal{G} \rightarrow \mathcal{R}$ }
is a function from all finite graphs $\mathcal{G}$ into a ring or field $\mathcal{R}$,
which is invariant under graph isomorphisms. 

In this paper we initiate the study of exact learnability of graph parameters
with values in $\mathcal{R}$, which is assumed to be either $\bZ,\bQ$ or $\bR$.
As this question seems new, we focus here on the special case of graph parameters
given as partition functions, \cite{ar:FreedmanLovaszSchrijver07,bk:Lovasz-hom}.
We adapt the model of exact learning introduced by D. Angluin \cite{ar:Angluin78}.
Our research extends the work of 
\cite{ar:BeimelBBKV00,pr:HabrardOncina06},
where exact learnability of languages (set of words or labeled trees) recognizable by
multiplicity automata (aka weighted automata) was studied, to graph parameters with values in $\cR$.

\subsection{Exact learning}
In each step, the learner may make {\em membership queries $\val(x)$} in which they ask 
for the value of the target $f$ on specific input $x$.
This is the analogue of the $\mem$ queries used in the original 
model of exact learning, \cite{ar:angluin87}.
The learner may also propose a hypothesis $h$ 
by sending an $\equ(h)$ query to the teacher.
If the hypothesis is correct, the teacher returns ``YES'' and 
if it is incorrect, the teacher returns a counterexample.
A class of functions is \emph{exactly learnable} if there 
is a learner that for each target function $f$, 
outputs a hypothesis $h$ such that $f(x) = h(x)$ for all $x$ and does so
in time polynomial in the size of a shortest representation 
of $f$ and the size of a largest counterexample returned by the teacher.

\subsection{Formulating a hypothesis}
To make sense one has to specify the formalism (language)  $\mathfrak{L}$ 
in which a hypothesis has to be formulated.
It will be obvious in the sequel, that the restriction imposed by the choice of $\mathfrak{L}$
will determine whether $f$ is learnable or not.

Let us look at the seemingly simpler case of learning integer 
functions $f: \bZ \rightarrow \bZ$ or integer valued functions of words $w \in \Sigma^\star$ over an alphabet in $\Sigma$. 
\begin{enumerate}[(i)]
\item
If $f$ can be any function $f: \bZ \rightarrow \bZ$ or $f: \Sigma^\star \rightarrow \bZ$, 
there are uncountably many candidate functions
as hypotheses, and no finitary formalism $\mathfrak{L}$ is suitable to
formulate a hypothesis.
\item
If $f$ is known to be a polynomial $p(X) = \sum_i a_i X^i \in \bZ[X]$,
we can formulate the hypothesis as a vector $\mathbf{a}= (a_1, \ldots, a_m)$
in $\bZ^m$. Learning is successful if the learner finds the hypothesis $h = \mathbf{a}$
in the required time. Here Lagrange interpolation will be used to formulate
the hypotheses.
\item
If $f$ is known to satisfy some recurrence relation, the hypothesis
will consist of the coefficients and the length of the recurrence relation,
and exact learnability will depend on the class of recurrence relations
one has in mind.
\item
If $f: \Sigma^\star \rightarrow \bZ$ is a word function recognizable by a multiplicity automaton $MA$,
the hypotheses are given by the weighted transition tables of $MA$, cf. \cite{ar:BeimelBBKV00}.
\end{enumerate}

Looking now at a graph parameter $f: \mathcal{G} \rightarrow \cR$ what can we expect?
Again we have to restrict our treatment to a class of parameters 
where each member can be described by a finite string in a formalism $\mathfrak{L}$.

We illustrate the varying difficulty of the learning problem with the example of the chromatic polynomial $\chi(G;X \in \bN[X])$ for a graph $G$.
For $X=k$, the evaluation of $\chi(G;k)$ counts the number of proper colorings of $G$
with at most $k$ colors. It is well known that for fixed $G$, $\chi(G;k)$ is indeed
a polynomial in $k$, \cite{ar:Birkhoff1912,bk:Bollobas99}.
A graph parameter $f$ is a {\em chromatic invariant over $\mathcal{R}$} if 
\begin{enumerate}[(i)]
\item
it is multiplicative, i.e., for the disjoint union $G_1 \sqcup G_2$ of $G_1$ and $G_2$, 
it holds that $f(G_1 \sqcup G_2) = f(G_1) \cdot f(G_2)$, and
\item
there are 
$\alpha, \beta, \gamma \in  \mathcal{R}$ such that
$f(G) = 
\alpha \cdot f(G_{-e}) + 
\beta \cdot f(G_{/e})  
\mbox{   and   }
f(K_1) = \gamma 
$.
\end{enumerate}
$K_n$ denotes the complete graph on $n$ vertices, and $G_{-e}$ and $G_{/e}$ are, respectively, the graphs obtained from deleting the edge $e$ from $G$ and contracting $e$ in $G$.

The parameter $\chi(G;k)$ is a chromatic invariant with $\alpha = 1, \beta = -1$ and $\gamma =k$.
Finally, $\chi(G;k)$ has an interpretation by counting homomorphisms:
\begin{gather*}
\chi(G;m) = \sum_{t: G \rightarrow K_m} 1,
\end{gather*}
This is a special case of the homomorphism counting function for a fixed graph $H$:
\begin{gather*}
\Hom(G,H) = \sum_{t: G \rightarrow H} 1,
\end{gather*}
where $t$ is a homomorphism $t: G \rightarrow H$.

Now, let a graph parameter $f: \cG \rightarrow \cR$ be the target of a learning algorithm.
\begin{enumerate}[(i)]
\item
If $f$ is known to be an instance of $\chi(G;X)$, a hypothesis consists of a value $X=a$.
But in this case we know that $\chi(K_1;X)=X$, so it suffices to ask for $f(K_1)=a$.
\item
If $f$ is known to be a chromatic invariant, the hypothesis consists of the
triple $(\alpha, \beta, \gamma)$.
In this case a hypothesis can be computed from the values of $f(P_m)$ for undirected paths
$P_m$ for sufficiently many values of $m$.
\item
If $f$ is known to be an instance of $\Hom(-,H)$, a hypothesis would consist of a target graph
$H$.
\end{enumerate}

\subsection{Counting weighted homomorphisms aka partition functions}
A {\em weighted graph $H(\alpha,\beta)$} is a 
graph $H =(V(H), E(H))$ on $n = |V(H)|$ vertices
together with a vertex weight function
$\alpha: V(H) \rightarrow \bR$, viewed as a vector of length $n$, 
and an edge weights function $\beta: V(H)^2 \rightarrow \bR$ viewed as an $n \times n$ matrix,
with $\beta(u,v)=0$ if $(u,v) \not \in E(H)$.

A {\em partition function}\footnote{
In the literature $\Hom(-, H(\alpha,\beta))$ is also denoted by $Z_{H(\alpha, \beta)}(G)$,
e.g., in \cite{ar:Sokal2005a}. We follow the notation of \cite{bk:Lovasz-hom}.
}
$\Hom(-, H(\alpha,\beta))$
is the generalization of $\Hom(-,H)$ to weighted graphs, whose value on a graph $G$ is defined as follows:
\begin{align*}
\Hom(G, H(\alpha, \beta)) = 
\sum_{t: G \rightarrow H} 
\prod_{v \in V(G)} \alpha(t(v))
\prod_{(u,v) \in V(G)^2} \beta(t(u),t(v))
\end{align*}
To illustrate the notion of a partition function, let $H_{indep}$ be the graph with two vertices $\{u,v\}$ and the edges $\{(u,v), (u,u)\}$, shown in Figure \ref{fig:ind}.
Let $\alpha(u)=1, \alpha(v)=X$ and $\beta(u,v)=1, \beta(u,u)=1$.
Then $\Hom(-, H_{indep}(\alpha,\beta))$ is the {\em independence polynomial},
\begin{gather*}
\Hom(G, H_{indep}(\alpha,\beta)) = I(G;X) = \sum_j ind_j(G) X^j
\end{gather*}
where $ind_j(G)$ is the number of independent sets of size $j$ in the graph $G$.

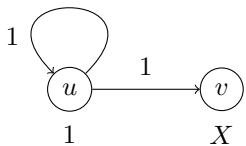
\begin{figure}%
\begin{tikzpicture}[scale=1]

\node[circle,draw] (v1) at (-2,1) {$u$};
\node[circle,draw] (v2) at (0,1) {$v$};
\draw[->]  (v1) edge (v2);
\draw  (v1) edge[->,loop, looseness=10] (v1);

\node (v1) at (-2,0.4) {$1$};
\node (v2) at (0,0.4) {$X$};
\node (v2) at (-1,1.3) {$1$};
\node at (-2.75,1.7) {$1$};
\end{tikzpicture}
\caption{The weighted graph $H_{indep}$.}%
\label{fig:ind}%
\end{figure}

We say a partition function $\Hom(-, H(\alpha,\beta))$ is {\em rigid} aka {\em asymmetric}
\footnote{
Some authors say $G$ is asymmetric if $G$ has no proper automorphisms, 
and $G$ is rigid if $G$ has no proper endomorphisms, \cite{kotters2009almost}.
Wikipedia uses rigid as we use it here.
}, 
if $H$ has no proper automorphisms. 
Note that automorphisms in a weighted graph 
also respect vertex and edge weights.
In our examples above, the evaluations of the independence polynomial are rigid partition functions,
whereas the evaluations of the chromatic polynomial are not.
It is known that almost all graphs are rigid:
\begin{theorem}[\cite{ar:ErdosRenyi1963,kotters2009almost}]
Let $G$ be a uniformly selected graph on $n$ vertices. 
The probability that $G$ is rigid tends to $1$ as $n \rightarrow \infty$.
\end{theorem}

If the target $f$ is known to be a (rigid) partition function 
$\Hom(-, H(\alpha,\beta))$
then the hypothesis consists of a (rigid) weighted graph $H(\alpha,\beta)$.

In Section \ref{sec:pre} we give the characterization of rigid and non-rigid partition functions
from \cite{ar:FreedmanLovaszSchrijver07,ar:lovasz06,bk:Lovasz-hom} in 
terms of connection matrices.

For technical reasons discussed in Section \ref{sec:conc},
in this paper we deal only with the learnability of rigid partition functions, and
leave the general case to future work.

\subsection{Main result}
Our main result can now be stated:

\begin{theorem}
\label{th:main}
Let $f$ be a graph parameter which is known to be a rigid partition function 
$f(G)= \Hom(G, H(\alpha,\beta))$.
Then $f$ can be learned in time polynomial in the size of $H$ and the size of the largest counterexample in 
the Blum-Shub-Smale model of computation over the reals with unit cost.
\end{theorem}
\begin{rem}
If $f$ takes values in $\bQ$ rather than in $\bR$ we can also work in the Turing model of computation
with logarithmic cost for the elements in $\bQ$.
\end{rem}

To prove Theorem \ref{th:main} we will use the characterization of rigid partition functions 
in terms of connection matrices, \cite[Theorem 5.54]{bk:Lovasz-hom}, stated as Theorem \ref{th:FLS} and Corollary \ref{cor:rigid}
in Section \ref{sec:pre}. The difficulty of our 
result lies not in finding a learning algorithm by carefully manipulating
the counterexamples to meet the complexity constraints, but in proving the algorithm correct.
In order to do this we had to identify and extract the suitable algebraic properties underlying
the proof of Theorem \ref{th:FLS} and Corollary \ref{cor:rigid}.

The learning algorithm is given in pseudo-code as Algorithm \ref{alg:full}. 
It maintains a matrix $M$ used in the generation of the hypothesis $h$ from $\val$ and $\equ$ query results.
After an initial setup of $M$, in each iteration the algorithm generates a hypothesis $h$,
queries the teacher for equivalence between $h$ and the target and either terminates, or updates $M$ accordingly and moves on to the next iteration.
\begin{algorithm}

\caption{Learning algorithm for rigid partition functions}
\label{alg:full}
\begin{algorithmic}[1]
\State $n = 1$
\While{True}
	\State $\aug(B_n)$
	\State $P = \findB(M)$
	\State $h = \hyp(P)$
	\If{$\equ(h) = \mathrm{YES}$} 
		\State \Return $h$
	\Else 
		\State $n = n+1$
		\State $B_{n} = \equ(h)$ \Comment{$B_n$ receives a counterexample}

	\EndIf
\EndWhile

\end{algorithmic}
\end{algorithm}

It uses three black-boxes; $\findB$ which uses $M$ to find a certain basis $P$ of a graph algebra associated with the target function (see Section \ref{sec:pre}), 
$\hyp$ which uses this basis and $\val$ queries to construct a hypothesis $h$, 
and \textsf{augment} $M$ which augments the matrix $M$ after a counterexample is received, using $\val$ queries.

We briefly overview the complexity of the algorithm to illustrate that rigid partition functions are indeed exactly learnable. 
Proofs of validity and detailed analysis of the complexity are given in later sections.
For a target $H(\alpha,\beta)$ on $q$ vertices, the procedure $\findB$ solves $O(q)$ systems of linear equations, 
and systems of linear matrix equations, all of dimension $O(\poly(q))$.
The procedure $\hyp$ performs $O(q)$ graph operations of polynomial time complexity on graphs of size $O(\poly(q,|x|))$, where $|x|$ is the size of the largest counterexample, and $O(q^2)$ $\val$ queries.
The procedure \textsf{augment} $M$ performs $O(q)$ $\val$ queries. Thus, each iteration takes time $O(\poly(q,|x|))$. 
Lemma \ref{lem:ind_counter} will show that there are $O(q)$ iterations, so the total run time of the algorithm is polynomial in the size $q$ of $H(\alpha,\beta)$ and the size $|x|$ of the
largest counterexample.

\paragraph*{Organization}
In Section \ref{sec:pre} we give the necessary background on partition functions and the graph algebras induced by them.
Section \ref{sec:algorithm} presents the algorithm in detail and in Section \ref{sec:valid} we prove its validity and analyze its time complexity. 
We discuss the results and future work in Section  \ref{sec:conc}. 
\ifappendix
Some of the more technical proofs appear in Appendix \ref{app:proofs}.
\else
Due to space limitations, the appendix is included in the arXiv version, \cite{arxiv:mfcs16}.
\fi 

\section{Preliminaries}
\label{sec:pre}

Let $k \in \bN$.
A \emph{$k$-labeled graph} $G$
is a finite graph in which $k$ vertices, or less, are labeled with labels from $[k] = \{1,\ldots,k\}$. 
We denote the class of $k$-labeled graphs by $\cG_k$.
The \emph{$k$-connection} of two
$k$-labeled graphs 
$G_1,G_2 \in \cG_k$ 
is given by taking the disjoint union of
$G_1$ and $G_2$ and identifying vertices with the same label.
This produces a $k$-labeled graph 
$G = G_1 G_2$.
Note that $k$-connections are commutative.

\subsection{Quantum graphs}

A formal linear combination of a finite number of $k$-labeled graphs $F_i \in \cG_k$ with coefficients 
from $\bR$ is called a \emph{$k$-labeled quantum graph}. $\cQ_k$ denotes
the set of $k$-labeled
quantum graphs.

Let $x,y$ be $k$-labeled quantum graphs: $x=\sum_{i=1}^{n}{a_i F_i}$, 
and $y=\sum_{i=1}^{n}{b_i F_i}$. Note that some of the coefficients may be zero.
$\cQ_k$ is an infinite dimensional vector space, with the operations:
$x + y = 
\left( \sum_{i=1}^{n}{a_i F_i} 
\right) 
+ \left( \sum_{i=1}^{n}{b_i F_i} \right) = 
\sum_{i=1}^{n}{(a_i + b_i) F_i}$, 
and 
$\alpha \cdot x = \sum_{i=1}^{n}{(\alpha a_i)F_i}$.

$k$-connections extend to $k$-labeled quantum graphs by
$
xy =
\sum_{i,j=1}^{n} (a_i  b_j) (F_i F_j)  
$. 
Any graph parameter $f$ extends to $k$-labeled quantum graphs linearly:
$ 
f(x)=
\sum_{i=1}^{n}{a_i f(F_i)}
$. 

\subsection{Equivalence relations for quantum graphs}

The \emph{$k$-connection matrix} $C(f,k)$ of a graph parameter $f: \cG \rightarrow \cR$ is a bi-infinite  
matrix over $\cR$ whose rows and columns are labeled with 
$k$-labeled graphs, and its entry at the row labeled with $G_1$ and the column labeled with $G_2$ contains the value of $f$ on $G_1 G_2$:
\begin{align*}
C(f,k)_{G_1, G_2} = f(G_1 G_2).
\end{align*}
Given a connection matrix $C(f,k)$, we associate with a $k$-labeled graph $G \in \cG_k$ the (infinite) row vector 
$R_G^k$ appearing in the row labeled by $G$ in $C(f,k)$. 
If $k$ is clear from context we write
$R_G$. 
Similarly, we associate an infinite row vector $R_x$ with $k$-labeled quantum graphs $x = \sum_{i=1}^{n}{a_i F_i}$, 
defined as $R_x = \sum_{i=1}^{n}{a_i R_{F_i}}$ 
where $R_{F_i}$ is the row in $C(f,k)$ labeled by the $k$-labeled graph $F_i$.

We say $C(f,k)$ has \emph{finite rank} if there are finitely many $k$-labeled graphs 
$\cB_{C(f,k)} = \{B_1,\ldots, B_n\}$ whose rows $\cR_{C(f,k)} = \{R_{B_1},\ldots, R_{B_n}\}$ linearly span $C(f,k)$.
Meaning, for any $k$-labeled graph $G$, there exists a linear combination of the rows in $\cR_{C(f,k)}$ which equals the row vector $R_G$. 
We say that $C(f,k)$ has rank $n$ and denote $r(f,k) = n$ if any set of less than $n$ graphs does not linearly span $C(f,k)$.

The main result we use is the characterization of partition functions in terms of connection matrices.
We do not need its complete power, so we state the relevant part:
\begin{theorem}[Freedman, Lov\'asz, Schrijver, \cite{ar:FreedmanLovaszSchrijver07}]
\label{th:FLS}
Let $f$ be a graph parameter that is equal to $\Hom(-,H(\alpha,\beta))$ for some $H(\alpha,\beta)$ on 
$q$ vertices. Then $r(f,k) \leq q^k$ for all $k \geq 0$.
\end{theorem}
The exact rank $r(f,k)$ was characterized in \cite{ar:lovasz06}, but first we need some definitions. 
A weighted graph $H(\alpha,\beta)$ is said to be \emph{twin-free} if $\beta$ does not contain two separate rows that are identical to each other
\footnote{ 
If $H(\alpha,\beta)$ has twin vertices, they can be merged into one vertex by adding their vertex weights 
without changing the partition function. As the size of the target representation is the smallest possible, 
we assume all targets are twin-free.}.
Let $H(\alpha,\beta)$ be a weighted graph
on $q$ vertices, 
and let $\Aut(H(\alpha,\beta))$ be the automorphism group of $H(\alpha,\beta)$. $\Aut(H(\alpha,\beta))$ acts on ordered $k$-tuples of vertices
$[q]^k = \{\phi: [k] \rightarrow [q]\}$ 
by $(\sigma \circ \phi)(i) = \sigma (\phi (i))$ for $\sigma \in \Aut(H(\alpha,\beta))$.
The \emph{orbit} of $\phi$ is the set of ordered $k$-tuples $\psi$ of vertices such that  $\sigma \circ \phi = \psi$ for an automorphism $\sigma \in \Aut(H(\alpha,\beta))$. 
The \emph{number of orbits} of $\Aut(H(\alpha,\beta))$ on $[q]^k$ is the number of different orbits for elements $\phi \in [q]^k$.
\begin{theorem}[Lov\'asz, \cite{ar:lovasz06}]
\label{th:orbits}
Let $f = \Hom(-, H(\alpha,\beta) )$ for a twin-free weighted graph $H(\alpha,\beta)$ on $q$ vertices. 
Then $r(f,k)$ is equal to the number of orbits of $\Aut(H(\alpha,\beta))$ on $[q]^k$ for all $k \geq 0$.
\end{theorem}
We use the special case:
\begin{cor}
\label{cor:rigid}
Let $f = \Hom(-, H(\alpha,\beta) )$ for a rigid twin-free weighted graph $H(\alpha,\beta)$ on 
$q$ vertices.
Then $r(f,k) = q^k$ for all $k \geq 0$.
\end{cor}

We define an equivalence relation $\equiv_{f,k}$ over $\cQ_k$ where two $k$-labeled quantum graphs $x$ and $y$ are in the same equivalence class 
if and only if the infinite vectors 
$R_x$
and $R_y$ are identical:
$
x \equiv_{f,k} y \iff R_x^k = R_y^k.
$
Note that the set $\cQ_k/f$ of equivalence classes of $\equiv_{f,k}$ is exactly the vector space 
$span(C(f,k))$ generated by linear combinations of rows in $C(f,k)$.
$k$-connections extend to these vectors by: $R_x R_y = R_{xy}$.

Thus, if $r(f,k) = n$ with spanning rows $\cR_{C(f,k)} = \{R_{B_1},\ldots, R_{B_n}\}$, they form a basis of $\cQ_k/f = span(C(f,k))$.
For brevity, we occasionally also refer to $\cB_{C(f,k)}$ as a basis.

Let $x$ be a $k$-labeled quantum graph whose equivalence class $R_x$ 
is given as the linear combination $R_x = \sum_{i=1}^{n}{\gamma_i R_{B_i}}$.
We call the column vector 
$\coef_x = (\gamma_1,\ldots,\gamma_n)^T$ the \emph{coefficients vector} of $x$, or \emph{representation} of $x$ using $\cB_{C(f,k)}$.

\section{The learning algorithm in detail}
\label{sec:algorithm}

In this section we present the learning algorithm in full detail. 
The commentary in this exposition  
foreshadows the arguments in Section \ref{sec:valid}, but otherwise validity is not considered here.
We do not address complexity concerns in this section either, however, we reiterate for the sake of clarity that 
the algorithm runs on a Blum-Shub-Smale machine, \cite{ar:bss89,bk:bss}, over the reals.
In such a machine, real numbers are treated as atomic objects; they are stored in single cells, and arithmetic operations are performed on them in a single step. 

{The objects the algorithm primarily works with are real matrices. In a context containing a basis $\cB_{C(f,k)}$, 
we associate a real matrix $A_x$ with each quantum graph $x$ such that the following holds. 
\begin{center}
The coefficients vector $\coef_{xy}$ of $xy$ using $\cB_{C(f,k)}$ is given by
$A_x \coef_{y}$.	\hspace{2em} (*)
\end{center}}
This device, as we will see in Section \ref{sec:valid}, will allow the algorithm to search for, and find, special quantum graphs that provide
a translation of the answers of $\val$ and $\equ$ queries
into a hypothesis.

As mentioned earlier, Algorithm \ref{alg:full} maintains a matrix $M$ which is a submatrix of $C(f,1)$.
In each iteration the algorithm generates a hypothesis $h = (\alpha^{(h)}, \beta^{(h)})$ using $M$,
and queries the teacher for equivalence between $h$ and the target $f$.
If the hypothesis is correct, the algorithm returns $h$, otherwise it augments $M$ with a $1$-labeled version of the counterexample, and moves on to the next iteration.

\begin{rem}
Strictly speaking, the teacher may be asked $\val$ queries on (unlabeled) graphs, however, we freely write
$\val(G)$ for $k$-labeled graphs $G \in \cG_k$.
Additionally, the algorithm will need to know
the value of the target on some quantum graphs. Since any graph parameter
extend to quantum graphs linearly, for a quantum graph $x = \sum_{i=1}^{n}{a_i F_i}$ we write
$\val(x)$ as shorthand for $\sum_{i=1}^{n}{a_i \cdot \val(F_i)}$ throughout the presentation.
\end{rem}

\paragraph*{Incorporating counterexamples}

The objective is to keep a non-singular submatrix $M$ of $C(f,1)$. 
The first $1$-labeled graph $B_1$ with which $M$ is augmented is some arbitrarily chosen $1$-labeled graph.

Upon receiving a $B_n$ graph as counterexample, 
the $1$-label is arbitrarily assigned to one of its vertices, making it a $1$-labeled graph.
Then $\aug(B_n)$ adds a row and a column to $M$ labeled with the (now) $1$-labeled graph $B_n$, 
and fills their entries with the values 
$f(B_n B_i) = f(B_i B_n)$, for $i \in [n]$, using $\val$ queries.

The other functions are slightly more complex.

\paragraph*{Finding an idempotent basis}

The function $\findB$, given in pseudo-code as Algorithm \ref{alg:find_basis}, receives as input the matrix $M$. 
For reasons which will become apparent later, we are interested in finding a certain (idempotent) basis of the linear space generated by
the rows of $C(f,1)$. For this purpose, in its first part $\findB$ iteratively, over $k=1,\ldots,n$, computes the entries of matrices  $A_{x}$ as in (*), where $x$ are $B_i$, $i \in [n]$, by
solving multiple systems $M \mathbf{x} = \mathbf{b}$ of linear equations,
and using the solutions $\Gamma$ of those systems to fill the entries of the matrices $A_{B_i}$, 
where the $(k,j)$ entry of $A_{B_i}$ is $\boldsymbol{\gamma}^{ij}(k)$.
Let $p_i$, $i \in [n]$ be those quantum graphs for which 
$A_{p_i}$ is the $n \times n$ matrix with the value $1$ in the entry $(i,i)$ and zero in all other entries.
Note that the matrices $A_{p_i}$, $i \in [n]$ are linearly independent. We will see that $p_i$, $i \in [n]$ are the idempotent basis, now we wish to find their
representation using $B_i$, $i \in [n]$.

For $i \in [n]$, the representation $\coef_{p_i}$ of the basic idempotent $p_i$ using the basis elements $B_i, i \in [n]$ is found by solving a system $A \mathbf{X} = A_{p_i}$
of linear \emph{matrix equations}, where $A$ is a block matrix whose blocks are the matrices $A_{B_i}$, $i \in [n]$.
Each solution is added to $\Delta$.
 
Finally, $\findB$ outputs the set $\Delta$ of these representations $\coef_{p_i}$, $i \in [n]$.
Then we have that $R_{p_i} = \sum_{k=1}^{n}{\coef_{p_i}(k) R_{B_k}}$ where $\coef_{p_i}$ is the 
coefficients vector of $p_i$ using $B_i$, ${i \in [n]}$.
The representations $\coef_{p_i} \in \Delta$ of the elements $p_i$, $i \in [n]$,
are what will provide a translation from results of $\val$ queries to weights.

\begin{algorithm}
\caption{$\findB$ function }
\label{alg:find_basis}
\begin{algorithmic}[1]
\State $\Gamma = \emptyset$
\For{each $i,j \in [n]$ }
	\For{$k=1,\ldots,n$}
		\State $\bsb(k) = \val(B_i B_j B_k)$
	\EndFor
	\State $\boldsymbol{\gamma}^{ij} = \solveLin(M \mathbf{x} =  \bsb)$
	\State $\Gamma = \Gamma \cup \{\boldsymbol{\gamma}^{ij}\}$
\EndFor
\For {$i \in [n]$}
		\State $A_{B_i} = \fillMatrix(i, \Gamma)$
		\State $A = \addBlock(A, i, A_{B_i})$ \Comment{$A$ is a block matrix with $A_{B_i}$ on its $i$th block}
\EndFor
\State $\Delta = \emptyset$
\For {$i \in [n]$}
	\State $\coef_{p_i} = \solveMat(A \mathbf{X} = A_{p_i})$
	\State $\Delta = \Delta \cup \{\coef_{p_i}\}$
\EndFor
\State \Return $\Delta$
\end{algorithmic}
\end{algorithm}

\paragraph*{Generating a hypothesis}
The function $\hyp$, given in pseudo-code as Algorithm \ref{alg:hyp}, receives as input the representations $\coef_{p_i}$ of the $1$-labeled quantum graphs $p_i$, $i \in [n]$,
which it uses to find the entries of the vertex weights vector $\alpha^{(h)}$ directly through $\val$ queries.

Then $\hyp$ finds the $2$-labeled analogues of these $1$-labeled quantum graphs. 
Those $2$-labeled analogues form a basis of 
of $\cQ_2/f$.

Denote by $K_2$ the $2$-labeled graph composed of a single edge with both vertices labeled.
Next, $\hyp$ finds the representation of $R_{K_2}^2$, that is the row labeled with $K_2$ in $C(f,2)$, using the basis $\cR_{C(f,2)}$. 
We find the representation of this specific graph $K_2$ as 
the coefficients in $\coef_{K_2}$ constitute the entries of the edge weights matrix $\beta^{(h)}$ (see Section \ref{sec:valid}).

This representation is found by solving a linear system of equations, similarly to how $\findB$ uses $\solveLin$,
but here we use the diagonal matrix $N$ whose entries correspond to the elements of $\cB_{C(f,2)}$. 

The solution of said system, i.e., the coefficients vector $\coef_{K_2}$ of $K_2$, 
is used to fill the edge weights matrix $\beta^{(h)}$.
If needed, $\beta^{(h)}$ is made twin-free by contracting the twin vertices into one and summing their weights in $\alpha^{(h)}$.

Finally, $\hyp$ returns the hypothesis $h = (\alpha^{(h)}, \beta^{(h)})$ as output.

\begin{algorithm}
\caption{$\hyp$ function }
\label{alg:hyp}
\begin{algorithmic}[1]
\For{each $i \in [n]$ }
	\State $\alpha^{(h)}(i) = \val(p_i)$
\EndFor
\State $N = 0^{n^2 \times n^2}$ \Comment{$N$ is a zero matrix of dimensions $n^2 \times n^2$.}
\For{$i=1,\ldots,n$}
	\For{$j=1,\ldots,n$}
		\State $p_{ij} = p_i \otimes p_j$ \Comment{See Remark \ref{rem:tensor}.}
		\State $N_{p_{ij},p_{ij}} = \val(p_{ij}p_{ij})$
		\State $\mathbf{b}(i_j) = \val( K_2 \, p_{ij})$
	\EndFor
\EndFor
\State $\beta^{(h)} = \solveLin(N \mathbf{x} =  \mathbf{b})$
\State $\twin(\alpha^{(h)},\beta^{(h)})$
\State $h = (\alpha^{(h)}, \beta^{(h)})$
\State \Return $h$
\end{algorithmic}
\end{algorithm}

\begin{rem}[Algorithm \ref{alg:hyp}]
\label{rem:tensor}
Let $q_i$ be the $1$-labeled quantum graph $p_i$ interpreted as a $2$-labeled quantum graph, and let $q_j$
be $p_j$ with the labels of its components renamed to $2$, and also interpreted as a $2$-labeled quantum graph.
The result of $p_i \otimes p_j$ is the $2$-labeled quantum graph $q_i \sqcup_2 q_j$.
\end{rem}

\section{Validity and complexity}
\label{sec:valid}

As stated earlier, a class of functions is exactly learnable if there 
is a learner that for each target function $f$, 
outputs a hypothesis $h$ such that $f$ and $h$ identify on all inputs, and does so
in time polynomial in the size of a shortest representation 
of $f$ and the size of a largest counterexample returned by the teacher.
The proof of Theorem \ref{th:main} argues that Algorithm \ref{alg:full} is such a learner for the class
of rigid partition functions, through Theorem \ref{th:valid}, which proves validity, and Theorem \ref{th:complexity}, which
proves the complexity constraints are met.

To prove validity, we first state existing results on properties of graph algebras induced by partition functions,
then show, through somewhat technical algebraic manipulations, how our algorithm successfully exploits these properties
to generate hypotheses. We then show our algorithm eventually terminates with a correct hypothesis.

For the rest of the section, let $H(\alpha,\beta)$ be a rigid twin-free weighted graph on $q$ vertices, 
and denote $f = \Hom(-, H(\alpha,\beta))$.
\begin{theorem}
\label{th:valid}
Given access to a teacher for $f$, Algorithm \ref{alg:full} outputs a hypothesis $h$ such that $f(G) = h(G)$ for all graphs $G \in \cG$.
\end{theorem}

The proof of the theorem follows from arguing that:
\begin{theorem}
\label{th:correct_hyp}
If $M$ is of rank $q$, then $\hyp$ outputs a correct hypothesis.
\end{theorem}
and that the rank of $M$ is incremented with every counterexample:
\begin{theorem}
\label{th:rank_n}
In the $\nth$ iteration of Algorithm \ref{alg:full} on $f$, $M$ has rank $n$.
\end{theorem}

First we confirm the hypotheses Algorithm \ref{alg:full} generates are indeed
in the class of graph parameters we are trying to learn, namely, rigid partition functions $\Hom(-,H(\alpha,\beta))$ for twin-free 
weighted graphs $H(\alpha,\beta)$.

Given Theorem \ref{th:rank_n}, for the hypothesis $h$ returned in the $\nth$ iteration,
the rank of $C(h,1)$ is at least $n$, since $M$ is a submatrix of $C(h,1)$. 
Thus, from Theorem \ref{th:orbits}, $h$ cannot have proper automorphisms, as it would imply that the rank of $C(h,1) < n$.
The fact that $h$ is twin-free is immediate from the construction in $\hyp$.

\subsection{From the idempotent bases to the weights - proof of \\ Theorem \ref{th:correct_hyp}}
Let $\cQ_k/f$ be of finite dimension $n$. The \emph{idempotent basis} $p_1,\ldots,p_n$ of $\cQ_k/f$ consists of 
those $k$-labeled quantum graphs $p_i$ for which $p_i p_i \equiv_{f,k} p_i$ and $p_i p_j \equiv_{f,k} 0$ for $i,j \in [n]$, $i \neq j$.
Recall how $\findB$ found those $1$-labeled quantum graphs $p_i$, $i \in [n]$ whose matrices $A_{p_i}$ behaved in this way.

In our setting of rigid twin-free weighted graphs, by \cite[Chapter 6]{bk:Lovasz-hom}, we have that
if $p_1,\ldots,p_q$ are the idempotent basis of $Q_1/f$, then the idempotent basis
of $\cQ_2/f$ is given by $p_i \otimes p_j$, $i,j \in [q]$. These are the $2$-labeled analogues mentioned in the description of $\hyp$.

Furthermore by \cite[Chapter 6]{bk:Lovasz-hom}, the vertex weights $\alpha$ of $H$ are given by $\alpha(i) = f(p_i)$, $i \in [q]$, 
and if the representation of $K_2$ using 
$p_i \otimes p_j$, $i,j \in [q]$ is $\sum_{i,j \in [q]}{\beta_{ij}(p_i \otimes p_j)}$, then the edge weights matrix $\beta$ is given by $\beta_{i,j} = \beta_{ij}$.
 
Equipped with these useful facts, we show that:
\begin{lemma}
\label{lem:idem}
If $M$ is of rank $q$, then $\findB$ outputs the idempotent basis of $\cQ_1/f$.
\end{lemma}
Then obtain Theorem \ref{th:correct_hyp} by showing how,
if $\hyp$ receives the idempotent basis of $\cQ_1/f$ as input, it outputs a correct hypothesis.

\paragraph*{Finding the idempotent basis - proof of Lemma \ref{lem:idem}}

Recall that in the presence of a basis $\cB_{C(f,k)}$ we associate a real matrix $A_x$ with each quantum graph $x$ such that the following holds. 
\begin{center}
The coefficients vector $\coef_{xy}$ of $xy$ using $\cB_{C(f,k)}$ is given by
$A_x \coef_{y}$.	
\end{center}

Let $B_i, B_j \in \cB_{C(f,1)}$, and denote by $\sum_{k=1}^{n}{\gamma^{i,j}_k R_{B_k}}$ the representation of the row $R_{B_iB_j}$ using $\cR_{C(f,1)}$, 
i.e., the row in $C(f,1)$
labeled with the graph resulting from the product $B_i B_j$.

\begin{claim}
\label{claim:Ax}
Let $x$ be some $1$-labeled quantum graph such that
$R_x = \sum_{i=1}^{n}{a_i R_{B_i}}$.
The matrix $A_x$ is given by
$(A_x)_{\ell,m} = \sum_{i=1}^{n}{a_i \gamma_\ell^{im}}$.
\end{claim}
Note that for a basis graph $B_k \in \cB_{C(f,1)}$, we have that
$
(A_{B_k})_{i,j} = \gamma_{i}^{k,j}
$.
\ifappendix
The proof of this claim appears in Appendix \ref{app:proofs}.
\else
The proof of this claim appears in the appendix of \cite{arxiv:mfcs16}.
\fi 

\begin{prop}
The matrices $A_{B_1},\ldots,A_{B_n}$  of the graphs in $\cB_{C(f,1)}$ are linearly independent and span all matrices of the form $A_x$
for a quantum graph $x$.
\end{prop}

If we know what are the matrices $A_{p_1},\ldots,A_{p_n}$ of the idempotent basis $p_1,\ldots,p_n$,
we can find their representation using $A_{B_1},\ldots,A_{B_n}$ by solving systems of linear matrix equations. 
Then, given a representation
$A_{p_i}  = \sum_{k=1}^{n}{\delta_k^{(i)} A_{B_k}}$,
we will have the representation of the basic idempotents using $\cB_{C(f,1)}$ as
$p_i = \sum_{k=1}^{n}{\delta_k^{(i)} B_k}$.

The definitions of $A_x$ and idempotence lead to the observation that for idempotent basics $p_i,p_j$, it holds that $A_{p_i} A_{p_i} = A_{p_i}$
and $A_{p_i} A_{p_j} = 0$.
From Corollary \ref{cor:rigid} we know the dimension of $\cQ_1/f$ is $q$, so we conclude:
\begin{prop}
\label{prop:id_matrices}
The idempotent basis for $\cQ_1/f$ consists of the quantum graphs $p_i$, $i \in [q]$ 
for which 
$A_{p_i}$ is the $q \times q$ matrix with the value $1$ in the entry $(i,i)$ and zero in all other entries.
That is,
\begin{align*}
A_{p_i}(k,j) = 
\begin{cases}
1, & \text{ if } (k,j) = (i,i)\\
0, & \text{ otherwise} 
\end{cases}
\end{align*}
\end{prop}

As $\findB$ solves the systems of linear matrix equations for these matrices, 
it remains to show that $\findB$ correctly computes the matrices $A_{B_i}$, $i \in [q]$.

Since $M$ is of full rank, the representations
$\sum_{k=1}^{n}{\gamma^{i,j}_k R_{B_k}}$ of graphs $B_i B_j$, $i,j\in [q]$ using $\cB_{C(f,1)}$
are correctly computed by the $\solveLin$ calls. And as noted before, the coefficients $\gamma^{i,j}_k$ are the entries of the matrices 
$A_{B_i}$, $i \in [q]$. Thus they indeed are correctly computed, and we have Lemma \ref{lem:idem}.

Since $\hyp$ directly queries the teacher for the values of $\alpha^{(h)}$, we have:
\begin{cor}
\label{cor:alpha}
If $M$ is of rank $q$, then $\hyp$ outputs a correct vertex weights vector $\alpha^{(h)}$.
\end{cor}
It remains to show this is true also for the edge weights:
\begin{prop}
\label{prop:beta}
If $M$ is of rank $q$, then $\hyp$ outputs a correct edge weights matrix $\beta^{(h)}$.
\end{prop}
\begin{proof}
As $p_{ij} = p_i \otimes p_j$, $i,j\in [q]$ are the idempotent basis for $\cQ_2/f$ we have that $p_{ij}p_{ij} \not\equiv_{f,2} 0$, so the matrix $N$
is a diagonal matrix of full rank, and $\solveLin$ indeed finds the representation of $K_2$ using $p_{ij}$, $i,j \in [q]$.
\end{proof}
From Corollary \ref{cor:alpha} and Proposition \ref{prop:beta} we have Theorem \ref{th:correct_hyp}.

Now we show that Algorithm \ref{alg:full} reaches that point in the first place.
\subsection{Augmentation results in larger rank - proof of Theorem \ref{th:rank_n}}
Theorem \ref{th:rank_n} is proved using the fact that $A_x$ are linearly independent for $k$-labeled quantum graphs
which are not equivalent in $\equiv_{f,k}$.
\begin{lemma}
\label{lem:ind_counter}
In the $\nth$ iteration of Algorithm \ref{alg:full}, if the teacher returns a counterexample $x$, then
$R_x$ is not spanned by $R_{B_1},\ldots,R_{B_n}$ where $B_1,\ldots,B_n$ are the graphs associated with the rows and columns of $M$.
\end{lemma}
\begin{proof}
If $n=1$, $M$ has rank $n$. Now let $M$ have rank $n$.

For contradiction, assume that $R_x = \sum_{i=1}^{n}{a_i R_{B_{i}}}$. Then $x \equiv_{f,1} \sum_{i=1}^{n}{a_i B_{i}}$ 
and we have that $\Hom(x,H) = \sum_{i=1}^{n}{a_i \Hom({B_{i},H)}}$ for the target graph $H$.
Denote by $h^{(n)}$ the hypothesis generated in this iteration.
If $x$ is a counterexample, it must hold that 
\begin{align*}
\Hom(x,h^{(n)}) \neq \Hom(x,H) = \sum_{i=1}^{n}{a_i \Hom(B_{i},H)} 
\end{align*}
The solution of the system of equations for $\bsb_x$ would give 
\begin{align*}
\Hom(x,h^{(n)}) = \sum_{i=1}^{n}{a_i \Hom(B_{i},h^{(n)})} = \sum_{i=1}^{n}{a_i \Hom(B_{i},H)}
\end{align*}
So we conclude that $\sum_{i=1}^{n}{a_i \Hom(B_{i},h^{(n)})} \neq \sum_{i=1}^{n}{a_i \Hom(B_{i},H)}$.

Since $M$ is of full rank, one can solve a system of linear equations using $M$ for $\mathbf{b}_x$ defined as $\mathbf{b}_x(k) = \val(x B_k)$, $k \in [n]$.
Now recall that the matrix $M$ contains \emph{correct} values $\Hom(B_{i}B_{j},H(\alpha,\beta))$, as it was augmented using $\val$ queries, 
therefore $M$ is a submatrix of $C(f,1)$. 
Thus the coefficients of the solution $\mathbf{a}$ of $M \mathbf{a} = \mathbf{b}_x$ equal $a_i$, $i \in [k]$, and we reach a contradiction.
Therefore we conclude $x \not\equiv_{f,1} \sum_{i=1}^{n}{a_i B_{i}}$ and its row $R_x$ is linearly independent from $R_{B_1},\ldots,R_{B_n}$.
\end{proof}
This also implies that the matrix $A_x$ associated with $x$ is not spanned by $A_{B_1},\ldots,A_{B_n}$. Therefore the submatrix of $C(f,1)$ composed of 
the entries of the rows and columns of ${B_1},\ldots,{B_n},x$ 
is of full rank $n+1$. This is exactly the matrix $M$ augmented with $x$, and we have Theorem \ref{th:rank_n}.
Combining this with Corollary \ref{cor:rigid}, we have:
\begin{cor}
\label{cor:number_of_iterations}
Let $f$ be a rigid partition function of a twin-free weighted graph on $q$ vertices.
Then Algorithm \ref{alg:full} terminates in $q$ iterations.
\end{cor}


\subsection{Complexity analysis}
As the algorithm runs on a  Blum-Shub-Smale machine for the reals and mostly solves systems of linear equations, it
is not difficult to show that it runs in time polynomial in the size of target and the largest counterexample.
First we observe:
\begin{prop}
\label{prop:connection}
Let $G_1,G_2 \in \cG_1$. Then $G_1 G_2$ can be computed in time $O(\poly(|G_1|,|G_2|))$.
\end{prop}
\begin{rem}
\label{rem:size_of_graphs}
$B_1$ is of fixed size, and all other $B_i$, $i =2,\ldots,n$, used in Algorithm \ref{alg:full} are counterexamples provided by the teacher, therefore they are all
of size polynomial in the size $|x|$ of the graph $x$.
\end{rem}

\begin{theorem}
\label{th:complexity}
Let $H(\alpha,\beta)$ be a rigid twin-free weighted graph on $q$ vertices and denote $f = \Hom(-,H(\alpha,\beta))$.
Given access to a teacher for $f$, Algorithm \ref{alg:full} terminates in time $O(\poly(q,|x|))$, where $|x|$ is the size
of the largest counterexample provided by the teacher.
\end{theorem}
\begin{proof}
From Corollary \ref{cor:number_of_iterations} it is enough to show that each iteration of Algorithm \ref{alg:full} does not
take too long (Lemma \ref{lem:complexity}).
\end{proof}

\ifappendix
\begin{lemma}
\label{lem:complexity}
\else
\begin{lemma}[\cite{arxiv:mfcs16}]
\label{lem:complexity}
\fi
In the $\nth$ iteration of Algorithm \ref{alg:full}, 
\textsf{augment} $M$, $\findB$, and $\hyp$ all
run in time $O(\poly(n,|x|))$.
\end{lemma}
\ifappendix
The easy proof is given in Appendix \ref{app:proofs}.
\else
\fi 

\begin{remark}
We note that, from \cite[Theorem $6.45$]{bk:Lovasz-hom}, 
the counterexamples provided by the teacher may be chosen to be of size at most $2(1+q^2)q^6$ where $q$ is the size of the target weighted graph.
\end{remark}

\section{Conclusion and future work}
\label{sec:conc}

This paper presented an adaptation of the exact model of learning of Angluin, \cite{ar:Angluin78}, to 
the context of graph parameters $f$ representable as partition functions of weighted graphs $H(\alpha,\beta)$.
We presented an exact learning algorithm for the class of \emph{rigid} partition functions 
defined by twin-free $H(\alpha,\beta)$.

If a weighted graph has proper automorphisms, its connection matrices $C(f,k)$ may have rank smaller than $q^k$. 
In this case, the translation from query results to a weighted graph would involve the construction of a submatrix of $C(f,k)$
for a sufficiently large $k$, and then find an idempotent basis for $\cQ_{k+1}/f$. 
We will study the learnability of non-rigid partition functions in a sequel to this paper.

Theorems similar to Theorem \ref{th:FLS} have been proved for variants of partition functions and connection matrices, 
\cite{ar:Schrijver2009,ar:Draisma-etal2012,ar:Schrijver-spin-2013,ar:schrijver15}.
It seems reasonable to us that similar exact learning algorithms exist for these settings,
but it is unclear how to modify our proofs here for this purpose.

\subparagraph*{Acknowledgements.}
We thank M. Jerrum and M. Hermann for their valuable remarks while listening to
an early version of the introduction of the paper, and A. Schrijver for his interest and encouragement. 
We also thank two anonymous referees for their helpful remarks.


\bibliography{partition_ref}

\appendix
\section{Proofs omitted from paper}
\label{app:proofs}
 
\subsection{Proof of Claim \ref{claim:Ax}}

For two graphs $B_i, B_j \in \cB_{C(f,1)}$, denote by $\sum_{k=1}^{n}{\gamma^{i,j}_k R_{B_k}}$ the representation of the row $R_{B_iB_j}$ using $\cR_{C(f,1)}$, 
i.e., the row 
labeled with the $1$-labeled graph resulting from the product $B_i B_j$.

Let $x,y$ be some $1$-labeled quantum graphs whose infinite row vectors are represented using $\cR_{C(f,1)}$ as
\begin{equation*}
R_x = \sum_{i=1}^{n}{a_i R_{B_i}} \ \ \ \ \ 
R_y = \sum_{j=1}^{n}{b_j R_{B_j}}
\end{equation*}
Then the representation of the row $R_{xy}$ of their product $xy$ is
\begin{align*}
R_{xy} &= \sum_{1 \leq i,j \leq n}{a_i b_j R_{B_i B_j}} 
= \sum_{1 \leq i,j \leq n}{a_i b_j \lb(\sum_{k=1}^{n}{\gamma^{i,j}_k R_{B_k}}\rb)} \\
&= \sum_{1 \leq i,j,k \leq n}{a_i b_j \gamma^{i,j}_k R_{B_k}}
\end{align*}
Thus the entry corresponding to the basis graph $B_k \in \cB_{C(f,1)}$ in the coefficients vector $\coef_{xy}$ is the scalar
\begin{align*}
\sum_{1 \leq i,j \leq n}{a_i b_j \gamma^{i,j}_k}.
\end{align*}

This scalar should equal the result of multiplying the $k$-th row of $A_x$ with the coefficients vector of $y$. 
Therefore the $k$-th row of $A_x$ would be
\begin{align*}
\lb(\sum_{i=1}^{n}{a_i \gamma^{i,1}_k}, \sum_{i=1}^{n}{a_i \gamma^{i,2}_k}, \ldots, \sum_{i=1}^{n}{a_i \gamma^{i,n}_k}\rb),
\end{align*}
Since then we would have:
\begin{align*}
\lb(\sum_{i=1}^{n}{a_i \gamma^{i,1}_k}, \ldots, \sum_{i=1}^{n}{a_i \gamma^{i,n}_k}\rb)
\lb(\begin{array}{c}
b_1 \\
\vdots \\
b_n
\end{array}\rb) = \sum_{j=1}^{n}{b_j\lb(\sum_{i=1}^{n}{a_i \gamma^{i,j}_k}\rb)} = \sum_{1 \leq i,j \leq n}{a_i b_j \gamma^{i,j}_k}
\end{align*}
Therefore the matrix $A_x$ is given by
\begin{align*}
A_x = \left(
\begin{array}{ccc}
\sum_{i=1}^{n}{a_i \gamma_{1}^{i,1}} & \cdots & \sum_{i=1}^{n}{a_i \gamma_{1}^{i,n}} \\
\vdots & & \vdots \\
\sum_{i=1}^{n}{a_i \gamma_{n}^{i,1}} & \cdots & \sum_{i=1}^{n}{a_i \gamma_{n}^{i,n}} 
\end{array}
\right)
\end{align*}

\subsection{Detailed complexity analysis - proof of Lemma \ref{lem:complexity}}
Let $H(\alpha,\beta)$ be a rigid twin-free weighted graph on $q$ vertices, 
and denote $f = \Hom(-, H(\alpha,\beta))$.
Let $|x|$ denote the size of the largest counterexample Algorithm \ref{alg:full} receives from the teacher.
\begin{lemma}
\label{lem:aug_complexity}
In the $\nth$ iteration of Algorithm \ref{alg:full}, \textsf{augment} $M$ runs in time
$O(\poly(n,|x|))$.
\end{lemma}

\begin{proof}
In the $\nth$ iteration, \textsc{augment} $M$ performs $O(n)$ $\val$ queries as it adds a new row and column labeled with $B_n$ to $M$.
For this is performs $\val$ queries on graphs that are $1$-connections between $B_n$ and $B_i$, $i \in [n]$. 
From Proposition \ref{prop:connection} and Remark \ref{rem:size_of_graphs}, it runs in time
$O(\poly(n,|x|))$.
\end{proof}

\begin{lemma}
\label{lem:findB_complexity}
In the $\nth$ iteration of Algorithm \ref{alg:full}, $\findB$ runs in time
$O(\poly(n,|x|))$.
\end{lemma}

\begin{proof}
$\findB$ has three $\mathbf{for}$ loops. In the first $\mathbf{for}$ loop, it repeats $O(n^2)$ times:
\begin{enumerate}
	\item 
	Fills an $n$-length vector $\mathbf{b}$ by making $\val(B_i B_j B_k)$ queries, for which it computes $B_i B_j B_k$. 
	Again from Proposition \ref{prop:connection} and Remark \ref{rem:size_of_graphs}, 
	we have that the computation of $\mathbf{b}$ in each of the $O(n^2)$ iterations is in time $O(\poly(n,|x|))$.
	\item
	Solves a linear system of equation of dimension $n$. This is in time $O(n^3)$.
\end{enumerate}

In each iteration of its second $\mathbf{for}$ loop, $\findB$ fills an $n \times n$ matrix and adds it to a block matrix, in time $O(n^2)$. This is repeated $n$ times.

In each iteration of its third $\mathbf{for}$ loop, $\findB$ solves a linear system of matrix equations involving $n \times n$ matrices, of dimension $n$. 
Such a system can be solved as a usual linear system of equations at the cost of a polynomial blowup 
where each matrix is replaced by $n^2$ variables, giving us time $O(n^6)$ for each of the $n$ iterations. 

In total, in the $\nth$ iteration of Algorithm \ref{alg:full}, $\findB$ runs in time $O(\poly(n,|x|))$.
\end{proof}

\begin{lemma}
\label{lem:hyp_complexity}
In the $\nth$ iteration of Algorithm \ref{alg:full}, $\hyp$ runs in time\\
$O(\poly(n,|x|))$.
\end{lemma}

\begin{proof}
Recall that for a quantum graph $x = \sum_{i=1}^{n}{ a_i F_i}$, we wrote $\val(x)$ as shorthand for $\sum_{i=1}^{n}{ a_i \val(F_i)}$.

All quantum graphs in the run are linear combinations of at most $n$ graphs, thus any linear combination
requires $O(n)$ arithmetic operations.

For the extraction of $\alpha^{(h)}$, $\hyp$ computes linear combinations of the results of $\val$ queries, $n$ times.

For the extraction of $\beta^{(h)}$, $\hyp$: 
\begin{enumerate}
	\item 
	Computes $p_{ij} = p_i \otimes p_j$ for $i,j \in[n]$. Each of these requires performing $2$-connections between
	$O(n^2)$ pairs of graphs of size $O(|x|)$, and the computation is performed for $O(n^2)$ indices $i,j$.
	\item
	Computes $p_{ij}p_{ij}$ for $i,j \in[n]$ and $\val(p_{ij}p_{ij})$, and computes
	$p_{ij}K_2$ and $\val(p_{ij}K_2)$. Each of these requires $O(n^4)$ operations on graphs of size $O(\poly(|x|))$. 
	The computation is performed for $O(n^2)$ indices $i,j$.
\end{enumerate}
In total, in the $\nth$ iteration of Algorithm \ref{alg:full}, $\hyp$ runs in time
$O(\poly(n,|x|))$.
\end{proof}

\end{document}